\documentclass[letterpaper]{article} 
\usepackage{aaai2027}  
\nocopyright
\usepackage[hyphens]{url}  
\usepackage{graphicx} 
\usepackage{natbib}  
\usepackage{caption} 
\usepackage{amsmath,amssymb,amsthm}
\usepackage{algorithm}
\usepackage{algorithmic}
 \usepackage{multirow}
\usepackage{newfloat}
\usepackage{listings}
\DeclareCaptionStyle{ruled}{labelfont=normalfont,labelsep=colon,strut=off} 
\floatstyle{ruled}
\newfloat{listing}{tb}{lst}{}
\floatname{listing}{Listing}

\usepackage{booktabs}

\newtheorem{lemma}{Lemma}
\newtheorem{proposition}{Proposition}

\title{Lapis: Laplacian Spiking Attention via First-Spike Timing and Membrane Leakage}
\author{
    Written by AAAI Press Staff\textsuperscript{\rm 1}\thanks{With help from the AAAI Publications Committee.}\\
    AAAI Style Contributions by Peter Patel Schneider,
    Sunil Issar,\\
    J. Scott Penberthy,
    George Ferguson,
    Hans Guesgen,
    Francisco Cruz\equalcontrib\corresponding,
    Marc Pujol-Gonzalez\equalcontrib\corresponding
}
\affiliations{
    \textsuperscript{\rm 1}Association for the Advancement of Artificial Intelligence\\

    1101 Pennsylvania Ave, NW Suite 300\\
    Washington, DC 20004 USA\\
    proceedings-questions@aaai.org
}

\title{Lapis: Laplacian Spiking Attention via First-Spike Timing and Membrane Leakage}
\author {
    Kaiwen Tang\textsuperscript{\rm 1},
    Jiaqi Zheng\textsuperscript{\rm 2},
    Zixuan Zhu\textsuperscript{\rm 1, 4, 5},
    Yiqun Wang\textsuperscript{\rm 3},
    Zhanglu Yan\textsuperscript{\rm 1}\corresponding,
    Weng-Fai Wong\textsuperscript{\rm 1}
}
\affiliations {
    \textsuperscript{\rm 1}School of Computing, National University of Singapore\\
    \textsuperscript{\rm 2}Sea AI Lab\\
    \textsuperscript{\rm 3}School of Electrical Engineering, Shanghai Jiaotong University\\
    \textsuperscript{\rm 4}Shanghai Advanced Research Institute, Chinese Academy of Sciences\\
    \textsuperscript{\rm 5}University of Chinese Academy of Sciences, Beijing\\
}

\begin{document}

\maketitle

\begin{abstract}
Self-attention has become central to spiking vision transformers, yet its query-key scoring is still largely inherited from dense networks. Existing spiking variants either simplify dot product scoring or replace it with discrete operators, but spike timing, the native variable of a spiking network, does not directly define how tokens are related.
We propose Lapis, a spiking attention mechanism that scores each token pair by the $\ell_1$ distance between its query and key first-spike latency vectors under time-to-first-spike coding, and maps this distance to an affinity through a Laplacian kernel. The kernel's exponential decay matches the impulse response of a leaky integrate-and-fire membrane, so the accumulated latency difference determines the decay of a membrane trace, while row normalization reduces to a bit shift under power-of-two rounding. Scoring therefore needs only subtraction, absolute value, and accumulation, and removes all multiplication between query and key channels. 
Under a matched backbone and training schedule, Lapis reaches 96.56\% top-1 accuracy on CIFAR-10, within 0.53 points of dot-product scoring. On ImageNet-1K, it reduces the estimated arithmetic energy of the attention path by $14.5\times$ relative to dense dot-product attention. The deployed 6-bit model attains 83.25\% top-1 accuracy at an estimated arithmetic energy of 3.28\,mJ per image.

\end{abstract}


\section{Introduction}
Spiking neural networks (SNNs) offer a brain-inspired and energy-efficient computing paradigm by processing information through sparse, event-driven spikes. ~\cite{hu2024toward}
With their growing success in visual recognition and other tasks, increasingly powerful architectures developed for artificial neural networks (ANNs) have been introduced into the spiking domain. 
In vision, the Vision Transformer (ViT) has made self-attention a central operation for recognition~\cite{touvron2021training, bao2021beit}, since attention relates every pair of tokens and captures the long-range dependencies that local operations miss. 
However, these architectures are primarily designed for dense, continuous-valued computation and therefore do not naturally align with the temporal representation and local dynamics of SNNs.

\begin{figure*}[t]
    \centering
    \includegraphics[width=\linewidth]{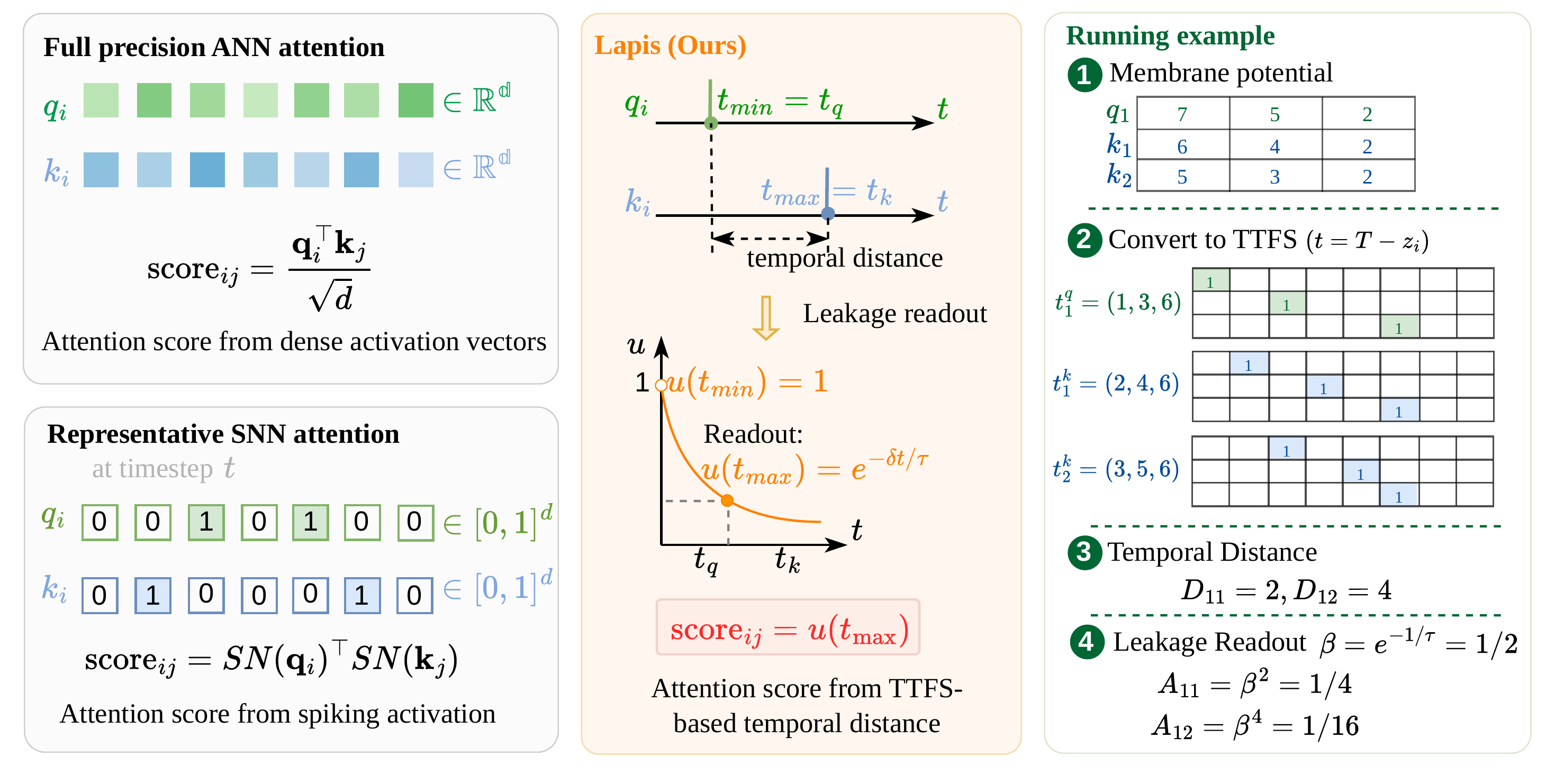}
    \caption{Rethinking query-key scoring with Lapis. ANN attention compares dense activations, while representative spiking attention compares binary states at one timestep. Lapis instead scores from first-spike timing: channel-wise latency differences are accumulated into a single temporal distance, which a leaky trace converts into an affinity by decaying for that duration.}
    \label{fig:lapis-concept}
\end{figure*}

This mismatch has motivated a range of spiking attention mechanisms that reduce the cost of standard self-attention. Some binarize queries, keys, and values to simplify their interactions~\cite{zhou2022spikformer}, while others replace pairwise multiplication with masking, addition, or spike-gated operations~\cite{yao2023spike, zhou2024qkformer}. 
Although effective in reducing arithmetic cost, these methods largely treat spikes as a computational constraint imposed on attention, rather than as a distinct representation of information. 
Their pairwise token similarity functions are therefore derived from simplified dot-product interactions or newly designed discrete operators~\cite{xiao2025rethinking, lee2025spiking}, while the temporal structure carried by spikes remains outside the relation function. This motivates a different design direction, which is defining token relations directly from spike timing. 


TTFS coding makes this opportunity explicit by representing each activation with its first-spike latency~\cite{park2020t2fsnn, goltz2021fast}. Nearby activation values produce nearby firing times, so the latency pattern can serve as a direct relation signal between tokens. Lapis therefore compares query and key tokens by the $\ell_1$ distance between their first-spike latency vectors, rather than by a dot product over encoded activations.
To convert temporal distance into interaction strength, Lapis applies a Laplacian kernel. Its $\ell_1$ distance accumulates channel-wise latency differences, while its exponential decay assigns larger affinities to token pairs with closer first-spike patterns. 
The same exponential form follows the subthreshold decay of a leaky integrate-and-fire (LIF) membrane, linking the attention affinity to membrane leakage dynamics. Thus, spike timing determines the relation, and membrane decay determines how strongly the tokens interact.
In this way, attention is no longer a dot-product operation imposed on spikes, but a temporal operation defined by the same dynamics that represent and process them.

We propose Lapis, a spiking attention mechanism in which token relations are defined by TTFS coding and implemented via membrane leakage dynamics. The module retains the standard query, key, and value structure but replaces dot-product similarity with a temporal-distance formulation. As a result, attention requires no multiplication between query and key channels, and row normalization reduces to a bit shift under power-of-two rounding.
It accumulates firing-time differences between queries and keys into a distance, which a Laplacian kernel then turns into interaction strength, its decay matching neuronal leakage dynamics. 
As a result, attention requires no dot products, and the exponential mapping is induced by the leakage process itself. By removing dense similarity computation, Lapis better matches the sparse nature of spike-based computation. Representation, computation, and neuronal dynamics are unified within a single temporal process.

 Lapis is evaluated as an attention operator rather than as a complete architecture. All models are trained with quantized activations and converted exactly to TTFS spiking networks. ImageNet models start from a pretrained vision transformer, and all ablations vary a single component under a fixed backbone and training schedule. Under this protocol, the Laplacian relation stays within 0.53 points of dot-product scoring on CIFAR-10, and replacing dot-product attention reduces the estimated energy of the attention path from 100.01 to 6.92\,mJ per image on ImageNet-1K. The converted Lapis-L model reaches 83.39\% top-1 at 13.20\,mJ, and 83.25\% at 3.28\,mJ once weights are quantized to 6 bits.

Our contributions are summarized as follows:
\begin{itemize}
\item We define token relations in spiking attention directly from first-spike latency vectors under TTFS coding. The resulting query-key score uses subtraction, absolute value, and accumulation, eliminating multiplication between query and key channels.

\item We propose Lapis, which maps temporal distance to attention affinity through a Laplacian kernel aligned with LIF membrane leakage. This design replaces dot-product score computation and softmax normalization with a spike-native temporal process.

\item We validate Lapis from CIFAR to ImageNet-1K. Lapis reaches 96.56\% top-1 accuracy on CIFAR-10 and 83.25\% on ImageNet-1K, while reducing the estimated arithmetic energy of the attention path by $14.5\times$ relative to dense dot-product attention under the same 45-nm operation-level model.

\end{itemize}

\section{Related Work}
\subsection{Visual Spiking Neural Networks}
Visual SNNs have become a practical route for energy efficient visual recognition. Their progress is commonly organized around ANN-to-SNN conversion and direct training with surrogate gradients~\cite{hu2024toward, tavanaei2019deep}. 
ANN-to-SNN conversion starts from a pretrained ANN and replaces continuous layers with spiking neurons. Combining with weights, thresholds, or membrane states calibration, the converted model approximates the source network over a certain simulation window. Recent conversion methods improve few-step inference by optimizing membrane initialization, shaping activation responses, or explicitly correcting residual conversion errors~\cite{bu2022optimized, bu2023optimal, huang2025converting}. 
Direct training learns SNNs with surrogate gradients directly instead, and its scalability has been improved by stabilizing spike propagation and adapting deep residual architectures to membrane or spike-domain communication~\cite{fang2021incorporating, zheng2021going, deng2022temporal, li2025rethinking}
Despite these advances, most visual SNNs still follow convolutional or residual backbones developed from ANNs, and their design mainly targets trainability and energy. They leave open how relations among visual tokens should be measured when information is represented by spike timing rather than dense activations.

\subsection{Spiking and Laplacian Transformers}
Spiking Transformer research has largely focused on adapting Transformer operations like self-attention to spike-based computation. Spikformer introduced spike-form Q, K, and V without softmax \cite{zhou2022spikformer}, while Spike-driven Transformer replaced their interactions with mask operations and sparse additions \cite{yao2023spike}. QKFormer later introduced linear-complexity Q-K attention \cite{zhou2024qkformer}, while Sorbet extended this line to spiking language models by approximating softmax and normalization with shift-based operations \cite{tang2024sorbet}. Recent variants revise either the interaction domain or the scoring rule. STAtten computes attention over spatiotemporal blocks, whereas $\alpha$-XNOR accounts for both spike and non-spike matches \cite{lee2025spiking, xiao2025rethinking}. Overall, existing methods still derive token relations from adapted attention operators, binary matching, or aggregate timing summaries. Using firing-time patterns across channels as the primary relation signal remains underexplored. Outside the spiking domain, EcoTransformer~\cite{gao2025ecotransformer} replaces dot-product scoring with an $\ell_1$-based Laplacian kernel, while LaplacianFormer~\cite{feng2026laplacianformer} further studies Laplacian kernels for efficient dense attention. These works establish Laplacian affinity as a viable relation function for conventional Transformers. Lapis instead formulates Laplacian attention over TTFS latencies.

\subsection{Time-to-First-Spike Coding}
Time-to-first-spike (TTFS) coding represents information by the latency of the first spike and is often implemented with at most one spike per neuron, making it attractive for sparse and low-latency inference~\cite{goltz2021fast}.
Early deep models focused on making this sparse code trainable. For example, T2FSNN combined kernel-based TTFS neurons with gradient optimization and early firing~\cite{park2020t2fsnn}, DTA-TTFS introduced dynamic firing thresholds and event-driven backpropagation~\cite{wei2023temporal}, and a recent formulation parallelized first-spike and gradient computation while improving output decoding~\cite{che2026parallel}.
TTFS has also been extended to Transformers through conversion and hardware co-design. TTFSFormer designs TTFS neurons for nonlinear Transformer layers~\cite{zhao2025ttfsformer}, whereas Otters realizes the temporal decay term with optoelectronic device dynamics~\cite{yan2025otters}. The remaining gap lies in attention rather than coding: first-spike times are used mainly to represent individual activations or implement layerwise computation, but rarely to compare tokens through firing-time patterns across feature channels.

\begin{figure*}[t]
    \centering
    \includegraphics[
        width=0.82\textwidth,
        keepaspectratio
    ]{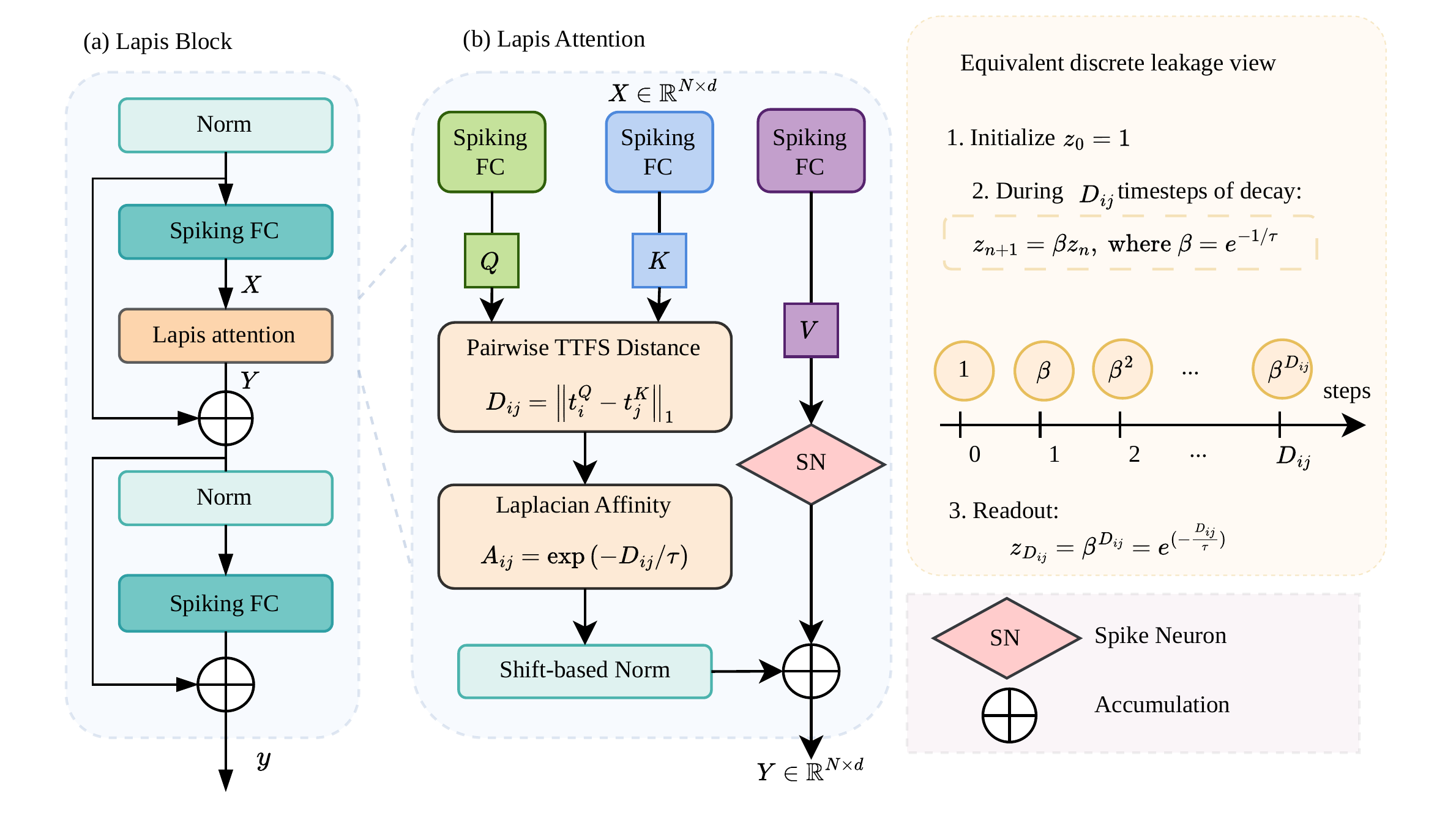}
    \caption{
        \textbf{Overview of the Lapis block and attention module.}
        A Lapis block follows a residual Transformer structure with Lapis attention and a spiking feed-forward network. Within Lapis attention, the query and key branches provide first-spike latencies whose pairwise $\ell_1$ distances are mapped to Laplacian affinities. A power-of-two row denominator enables shift-based normalization. The value branch passes through an independent spiking neuron before weighted accumulation.
        The inset shows the equivalent discrete leakage view: initializing a unit state and applying $D_{ij}$ decay steps yields $\beta^{D_{ij}}$, equal to the corresponding Laplacian affinity under a unit simulation step.
    }
    \label{fig:lapis-overview}
\end{figure*}

\section{Preliminaries}
\subsection{Leaky Integrate-and-Fire Model}
The leaky integrate-and-fire (LIF) model describes a neuron by 
its membrane potential \(u(t)\), which integrates incoming input 
while decaying toward a resting potential \(u_{\mathrm{rest}}\). 
Its subthreshold dynamics are given by
\begin{equation}
\tau_m \frac{\mathrm{d}u(t)}{\mathrm{d}t}
=
-\bigl(u(t)-u_{\mathrm{rest}}\bigr) + I(t),
\end{equation}
where \(I(t)\) denotes the input drive and \(\tau_m>0\) is the 
membrane time constant. When the membrane potential reaches a 
firing threshold, the neuron emits a spike and its potential is 
reset. In the absence of input and threshold crossings, the 
membrane potential follows
\begin{equation}
u(t+\Delta)-u_{\mathrm{rest}}
=
\bigl(u(t)-u_{\mathrm{rest}}\bigr)
\exp\left(-\frac{\Delta}{\tau_m}\right).
\end{equation}
The membrane time constant therefore determines how quickly the 
effect of a past input diminishes with elapsed time
\cite{gerstner2014neuronal}.

\subsection{Time-to-First-Spike Coding}
Time-to-first-spike (TTFS) coding represents a neuronal response 
by the latency of its first spike relative to the beginning of a 
finite coding window. Only the first spike contributes to the 
representation; later spikes do not change the encoded value and 
can therefore be omitted. TTFS can thus represent an active 
response with a single event, whereas rate coding represents 
information through the number of spikes observed over time. A 
neuron that does not emit a spike within the coding window remains 
silent. The mapping between a value and its spike latency is 
determined by the encoding rule rather than by TTFS itself
\cite{park2020t2fsnn,goltz2021fast}.

\section{Method}

\subsection{Overview}
The central idea of Lapis is to form attention weights from spike timing through neuronal leakage. Lapis retains the multi-head query, key, and value projections. For head $h$, let $\mathbf{t}_{i}^{Q,h}$ and $\mathbf{t}_{j}^{K,h}$ denote the channel-wise first-spike latency vectors of query token $i$ and key token $j$, respectively. We define their temporal distance as
\begin{equation}
    D_{ij}^{h}
    =
    \left\|
    \mathbf{t}_{i}^{Q,h}
    -
    \mathbf{t}_{j}^{K,h}
    \right\|_{1},
    \label{eq:overview-distance}
\end{equation}
where a smaller $D_{ij}^{h}$ indicates closer agreement between the two first-spike patterns.

We convert the temporal distance into a positive, unnormalized affinity using a Laplacian kernel:
\begin{equation}
    A_{ij}^{h}
    =
    \exp\left(
    -\frac{D_{ij}^{h}}{\tau_h}
    \right),
    \label{eq:overview-affinity}
\end{equation}
where $\tau_h>0$ is a pre-head learnable temporal scale. A smaller $\tau_h$ causes the affinity to decay more rapidly with temporal distance, whereas a larger $\tau_h$ permits larger differences in first-spike timing. The exponential decay follows the subthreshold dynamics of an LIF membrane, providing a leakage-based realization of the temporal affinity.
For each query, Lapis accumulates the affinities over all keys
and quantizes the resulting normalization factor to a power of
two. Under fixed-point arithmetic, the final scaling can
therefore be implemented by a bit shift.

Overall, Lapis forms attention weights from an $\ell_1$ temporal distance through leakage-based decay and power-of-two normalization. Computing the query--key distance requires only subtraction, absolute value, and accumulation over channels, avoiding the channel-wise multiplications required by dot-product scoring.

\subsection{Temporal Query and Key Relations}
\label{sec:temporal-relations}
We now define the latency vectors used in
Eq.~\eqref{eq:overview-distance}. For each head $h$, let
\begin{equation}
    S^{Q,h}, S^{K,h} \in \{0,1\}^{T \times N \times C}
\end{equation}
be the query and key spike tensors, where $T$ is the simulation
length, $N$ is the number of tokens, and $C$ is the channel dimension
of the head. We take one simulation step as the unit of latency.

For any token spike train $X\in\{0,1\}^{T\times C}$, define the
first-spike map $\phi(X)\in\{0,\ldots,T\}^{C}$ by
\begin{equation}
    [\phi(X)]_c
    =
    \begin{cases}
        \min\{n \mid X_{n,c}=1,\ 0\le n<T\},
        & \text{if such } n \text{ exists},\\
        T,
        & \text{otherwise}.
    \end{cases}
    \label{eq:first-spike-map}
\end{equation}
The terminal value $T$ denotes no spike within the simulation window.
In our TTFS conversion, it is the code for zero activation.

The query and key latency vectors are
\begin{equation}
    \mathbf{t}_{i}^{Q,h}
    =
    \phi(S^{Q,h}_{:,i,:}),
    \qquad
    \mathbf{t}_{j}^{K,h}
    =
    \phi(S^{K,h}_{:,j,:}).
    \label{eq:qk-latency-vectors}
\end{equation}
Lapis defines their temporal relation by the $\ell_1$ distance between
the two latency code vectors as Equation~\ref{eq:overview-distance}.
Thus, a zero-code entry is treated like any other latency code: if the query and key entries are both $T$, their contribution to the distance is zero. The next subsection maps $D_{ij}^{h}$ to a positive affinity through membrane leakage.

\subsection{Laplacian Affinity from Membrane Leakage}
\label{sec:leakage-affinity}

Given the temporal distance $D_{ij}^{h}$, Lapis assigns an affinity to each query--key pair by
\begin{equation}
    A_{ij}^{h}
    =
    \exp\left(
        -\frac{D_{ij}^{h}}{\tau_h}
    \right),
    \label{eq:lapis-affinity}
\end{equation}
where $\tau_h>0$ is the temporal scale of head $h$. A smaller $\tau_h$ gives a sharper decay with temporal distance, while a larger $\tau_h$ allows broader timing variation. Channels enter the affinity only through the additive $\ell_1$ reduction. Lapis first accumulates the scalar $D^h_{ij}$ and then maps it to the single affinity $A^h_{ij}$.

The exponential form in Eq.~\eqref{eq:lapis-affinity} matches the subthreshold decay of an LIF membrane. Let
\begin{equation}
    \beta_h
    =
    \exp\left(
        -\frac{1}{\tau_h}
    \right),
    \qquad
    0<\beta_h<1,
    \label{eq:leak-factor}
\end{equation}
and consider the discrete leaky recurrence
\begin{equation}
    u_{r+1}
    =
    \beta_h u_r,
    \qquad
    u_0=1,
    \label{eq:discrete-leakage}
\end{equation}
with no input and no reset.

\begin{proposition}[Leakage realization of Lapis affinity]
\label{prop:leakage-realization}
For any query-key pair $(i,j)$, the affinity in Eq.~\eqref{eq:lapis-affinity} is obtained by running the leaky recurrence in Eq.~\eqref{eq:discrete-leakage} for $D_{ij}^{h}$ steps:
\begin{equation}
    u_{D_{ij}^{h}}
    =
    A_{ij}^{h}.
    \label{eq:leakage-realization}
\end{equation}
\end{proposition}

\begin{proof}
By Eq.~\eqref{eq:discrete-leakage}, $u_r=\beta_h^r$ for any
nonnegative integer $r$. Taking $r=D_{ij}^{h}$ and substituting
Eq.~\eqref{eq:leak-factor} gives
\begin{equation}
    u_{D_{ij}^{h}}
    =
    \beta_h^{D_{ij}^{h}}
    =
    \exp\left(
        -\frac{D_{ij}^{h}}{\tau_h}
    \right)
    =
    A_{ij}^{h}.
\end{equation}
\end{proof}

The proposition shows that the Lapis affinity can be evaluated
by applying the membrane-leak recurrence for
$D_{ij}^{h}$ updates.

\begin{proposition}[Properties of the temporal affinity]
\label{prop:affinity-properties}
For every head $h$ and all token pairs $(i,j)$:

(i) \emph{Finite affinity set.}
Since every first-spike latency lies in $\{0,\ldots,T\}$, we have $D_{ij}^{h}\in\{0,\ldots,CT\}, \text{and} A_{ij}^{h}\in\{\beta_h^{r}: r=0,\ldots,CT\},$
so $A_{ij}^{h}$ takes at most $CT+1$ distinct values per head.

(ii) \emph{Positivity and order consistency.}
$A_{ij}^{h}\in(0,1]$, with $A_{ij}^{h}=1$ if and only if
$\mathbf{t}_{i}^{Q,h}=\mathbf{t}_{j}^{K,h}$. Moreover, for any keys
$j,\ell$,
\[
A_{ij}^{h}>A_{i\ell}^{h}
\;\Longleftrightarrow\;
D_{ij}^{h}<D_{i\ell}^{h},
\]
i.e., the affinity is strictly decreasing in the temporal distance.

(iii) \emph{Latency stability.}
If the query and key latencies are perturbed by
$\boldsymbol{\delta}_{i}^{Q,h}$ and $\boldsymbol{\delta}_{j}^{K,h}$ with
$\|\boldsymbol{\delta}_{i}^{Q,h}\|_1\le\eta_Q$ and
$\|\boldsymbol{\delta}_{j}^{K,h}\|_1\le\eta_K$, then
\[
\bigl|\widetilde A_{ij}^{h}-A_{ij}^{h}\bigr|
\le
\frac{\eta_Q+\eta_K}{\tau_h}.
\]
The proof is given in supplementary material.
\end{proposition}

\subsection{Lapis Attention}
\label{sec:lapis-attention}

The affinities in Eq.~\eqref{eq:lapis-affinity} determine the relative
importance of the keys, but their row sums can vary across queries.
Using them directly for value aggregation would therefore make the
output magnitude depend on both the relative affinities and their total
mass. Lapis controls this variation by scaling each affinity row before
value aggregation. For query token $i$, let
\begin{equation}
    Z_i^{h}
    =
    \sum_{j=1}^{N} A_{ij}^{h}.
    \label{eq:lapis-partition}
\end{equation}

Exact row normalization requires division by the general positive value
$Z_i^{h}$. We instead approximate it by the nearest power of two:
\begin{equation}
    k_i^{h}
    =
    \operatorname{round}\!\left(\log_2 Z_i^{h}\right),
    \qquad
    \widehat Z_i^{h}
    =
    2^{k_i^{h}}.
    \label{eq:pot-partition}
\end{equation}
The Lapis attention weight is then
\begin{equation}
    \widehat W_{ij}^{h}
    =
    \frac{A_{ij}^{h}}{\widehat Z_i^{h}}
    =
    A_{ij}^{h}2^{-k_i^{h}}.
    \label{eq:lapis-weight}
\end{equation}
Since $k_i^{h}$ is an integer, the scaling in
Eq.~\eqref{eq:lapis-weight} can be implemented by a bit shift under
fixed-point arithmetic.

For comparison, exact row normalization gives
\begin{equation}
    W_{ij}^{*,h}
    =
    \frac{A_{ij}^{h}}{Z_i^{h}}
    =
    \operatorname{softmax}_{j}
    \left(
        -\frac{D_{ij}^{h}}{\tau_h}
    \right).
    \label{eq:exact-temporal-attention}
\end{equation}
The following result bounds the approximation introduced by
power-of-two scaling.

\begin{lemma}[Power-of-two approximation]
\label{lem:pot-approximation}
For every query $i$, key $j$, and head $h$,
\begin{equation}
    \frac{1}{\sqrt{2}}W_{ij}^{*,h}
    \leq
    \widehat W_{ij}^{h}
    \leq
    \sqrt{2}W_{ij}^{*,h}.
    \label{eq:pot-approximation}
\end{equation}
\end{lemma}

Since
$\left|k_i^{h}-\log_2 Z_i^{h}\right|\leq 1/2$, we have
\[
    2^{-1/2}
    \leq
    \frac{Z_i^{h}}{\widehat Z_i^{h}}
    \leq
    2^{1/2},
\]

which gives Lemma~\ref{lem:pot-approximation}. A complete derivation is provided in supplementary material. The same scaling factor is applied to every key associated with a query. Consequently, power-of-two scaling preserves the relative affinities and their ordering, while keeping the row mass within $[1/\sqrt{2},\sqrt{2}]$.

The value branch remains a spike train. Let $\mathbf S_{n,j}^{V,h}\in\{0,1\}^{C}$ denote the value spike vector of token $j$ at step $n$. Lapis aggregates the value spikes as
\begin{equation}
    \mathbf Y_{n,i}^{h}
    =
    \sum_{j=1}^{N}
    \widehat W_{ij}^{h}
    \mathbf S_{n,j}^{V,h},
    \qquad
    0\leq n<T.
    \label{eq:lapis-value-aggregation}
\end{equation}
The same token-mixing weights are applied at every simulation step, so
the temporal structure of the value spikes is retained. The outputs of
all heads are then concatenated.

\begin{table*}[t]
\centering
\small
\setlength{\tabcolsep}{6.2pt}
\renewcommand{\arraystretch}{1.05}
\begin{tabular}{@{}lllrcc@{}}
\toprule
Model
& Type
& Architecture
& Energy (mJ) $\downarrow$
& Time step $T$
& Top-1 Acc. (\%) $\uparrow$ \\
\midrule
DeiT~\cite{touvron2021training}
& ANN
& DeiT-B
& 80.50
& --
& 81.80 \\
PVT~\cite{wang2021pyramid}
& ANN
& PVT-Large
& 45.08
& --
& 81.70 \\
BEiT~\cite{bao2021beit}
& ANN
& ViT-B/16
& --
& --
& 83.20 \\
BEiT~\cite{bao2021beit}
& ANN
& ViT-L/16
& --
& --
& 85.20 \\
\midrule
Spikformer~\cite{zhou2022spikformer}
& DT-SNN
& Spikformer-8-768
& 21.48
& 4
& 74.81 \\
CML~\cite{zhou2023enhancing}
& DT-SNN
& Spikingformer-8-768
& 16.30
& 4
& 77.64 \\
\multirow{2}{*}{QKFormer~\cite{zhou2024qkformer}}
& \multirow{2}{*}{DT-SNN}
& \multirow{2}{*}{HST-10-768}
& 8.52
& 1
& 81.69 \\
&
&
& 38.91
& 4
& 84.22 \\
\multirow{2}{*}{Max-Former~\cite{fang2026spiking}}
& \multirow{2}{*}{DT-SNN}
& \multirow{2}{*}{Max-10-768}
& 5.27
& 1
& 78.60 \\
&
&
& 14.87
& 4
& 82.39 \\
MST~\cite{wang2023masked}
& A2S
& Swin-T (BN)
& --
& 512
& 78.51 \\
SpikeZIP-TF~\cite{you2024spikezip}
& A2S
& SViT-L-32Level
& $1{,}270.4^{\dagger}$
& 64
& 83.82 \\
\midrule
\multirow{4}{*}{\textbf{Lapis}}
& \multirow{4}{*}{A2S}
& \textbf{Lapis-B}
& 1.50
& 15
& 80.32 \\
&
& \textbf{Lapis-B (FP)}
& 5.35
& 15
& 80.58 \\
&
& \textbf{Lapis-L}
& 3.28
& 20
& 83.25 \\
&
& \textbf{Lapis-L (FP)}
& 13.20
& 20
& 83.39 \\
\bottomrule

\end{tabular}
\caption{ImageNet-1K classification at $224\times224$.
DT-SNN denotes direct training, while A2S denotes ANN-to-SNN
conversion. Energy is the estimated per-image inference energy under the operation-level models used by the corresponding studies. For Lapis, $T$ is the length of the TTFS coding window, with at most one spike per neuron, rather than the number of recurrent simulation steps. $^{\dagger}$Derived from SpikeZIP-TF's reported 19.85\,W using $64$ steps of 1\,ms.}
\label{tab:imagenet_main}
\end{table*}

\subsection{Quantized Training and TTFS Conversion}
\label{sec:qnn-ttfs-conversion}

Lapis is trained as a quantized neural network and converted to a TTFS SNN for inference. This avoids backpropagation through discrete first-spike events. During training, the quantized query and key representations are processed by the same Lapis operator defined in the preceding subsections.

Let $z\in\{0,\ldots,T\}$ be an integer activation code. We convert it
to a binary spike train $\mathcal{E}(z)\in\{0,1\}^{T}$ as
\begin{equation}
    [\mathcal{E}(z)]_n
    =
    \mathbb{I}\!\left[z>0 \ \land\ n=T-z\right],
    \qquad
    0\leq n<T,
    \label{eq:qnn-ttfs-encoding}
\end{equation}
where $\mathbb{I}[\cdot]$ is the indicator function. Each positive
activation code therefore produces one spike at its assigned time
step, and larger codes produce earlier spikes. The encoding is applied
elementwise to the quantized activations throughout the network.

By construction, the first-spike time of the encoded activation is
$\mathcal{T}(\mathcal{E}(z))=T-z$. The conversion therefore preserves
the query--key distance:
\begin{equation}
\begin{aligned}
    \left\|
        \mathcal{T}\!\left(\mathcal{E}(\mathbf q_i^{h})\right)
        -
        \mathcal{T}\!\left(\mathcal{E}(\mathbf k_j^{h})\right)
    \right\|_1
    &=
    \left\|
        (T-\mathbf q_i^{h})
        -
        (T-\mathbf k_j^{h})
    \right\|_1\\
    &=
    \left\|
        \mathbf q_i^{h}
        -
        \mathbf k_j^{h}
    \right\|_1.
\end{aligned}
\label{eq:qnn-ttfs-consistency}
\end{equation}

With the same temporal scale and power-of-two scaling rule, the QNN and the converted SNN produce identical Lapis affinities and attention weights.

\section{Experiment}
\subsection{Experimental Setup}

\paragraph{Datasets and models.}
We evaluate Lapis on CIFAR-10, CIFAR-100, and ImageNet-1K and
report top-1 classification accuracy. We instantiate Lapis at three scales. Lapis-S is used for the CIFAR experiments and all controlled ablations, while Lapis-B and Lapis-L follow the ViT-B and ViT-L configurations and are evaluated on ImageNet-1K. The TTFS coding windows are set to $T=15$ for Lapis-S and Lapis-B, and $T=20$ for Lapis-L. Under the 6-bit deployment setting, the model sizes of Lapis-S are 18.0\,MB on CIFAR-10 and 18.2\,MB on CIFAR-100, while those of Lapis-B and Lapis-L are 70.7\,MB and 238.1\,MB, respectively. All ablation variants use the same Lapis-S backbone and training configuration, with only the component under study being changed.

\paragraph{Training and conversion.}
CIFAR models are trained for 100 epochs with a batch size of 128 using AdamW, an initial learning rate of $1\times10^{-4}$, a weight decay of 0.05, and a cosine schedule with a 5-epoch linear warmup. Unless otherwise specified, CIFAR training uses both logit- and feature-level
knowledge distillation. For ImageNet-1K, we initialize the backbone from a pretrained ViT, namely BeiT~\cite{bao2021beit} and fine-tune it for 60 epochs with a global batch size of 162 and an initial learning rate of $8.5\times10^{-5}$. Each deployment stage is subsequently recovered by a short quantization-aware fine-tuning process with a batch size of 108 and a learning rate of $2\times10^{-5}$. Following Section~\ref{sec:qnn-ttfs-conversion}, the resulting
integer activations are mapped to TTFS spikes for SNN inference.

\paragraph{Implementation.}
Our models are implemented in PyTorch 2.12 with CUDA 12.8 and trained on NVIDIA RTX PRO 6000 Blackwell Server Edition GPUs with 96\,GB of memory. CIFAR experiments use a single GPU, whereas ImageNet experiments use distributed data parallel training on up to three GPUs. We use bfloat16 automatic mixed precision, while the $\ell_1$ distance, exponential affinity, and normalization operations are evaluated in FP32 for numerical stability. Further training and quantization details are provided in the supplementary material.

\subsection{Main Result}
Table~\ref{tab:imagenet_main} compares Lapis with representative dense, directly trained, and converted models. Lapis-B and Lapis-L achieve 80.58\% and 83.39\% top-1 accuracy using 15-step and 20-step TTFS coding windows, respectively. The FP32 Lapis-L model requires an estimated 13.20\,mJ per image. Quantizing its body weights to 6 bits reduces the energy to 3.28\,mJ while retaining 83.25\% accuracy, corresponding to a $4.03\times$ energy reduction. The FP32 Lapis-L model improves over Max-Former at $T=4$ by 1.00 percentage point while using 11.2\% less energy. It reports 66.1\% lower estimated energy than QKFormer at $T = 4$. The 6-bit model extends this to 3.28 mJ, the lowest estimated inference energy among the evaluated models above 83\% top-1, and $11.9\times$ below QKFormer at $T = 4$.

Among conversion-based methods, it matches SpikeZIP-TF to within 0.57 points while using a TTFS coding window as 20 with single spike representation instead of 64 multi-spike simulation steps. 
Together, these results show that Lapis combines large-scale recognition accuracy with a compact temporal representation and low arithmetic cost.

Table~\ref{tab:cifar_main} reports Lapis-S on CIFAR-10 and CIFAR-100. With full-precision body weights, Lapis-S reaches 96.56\% and 81.41\%, the highest accuracy among the compared models on both benchmarks, exceeding QKFormer by 0.38 and 0.26 points and Spikformer by 1.05 and 3.20 points. Quantizing the body weights to 6 bits costs 0.45 and 0.39 points, and the resulting deployment configuration stays within 0.13 points of QKFormer on both benchmarks. The same relation function and low-precision deployment path therefore transfer across model scales.
\begin{table}[t]
\centering
\small
\setlength{\tabcolsep}{3pt}
\renewcommand{\arraystretch}{1.05}
\begin{tabular}{lccc}
\toprule
Method & CIFAR-10 & CIFAR-100 \\
\midrule
Spikformer~\citeyearpar{zhou2022spikformer}  & 95.51 & 78.21 \\
SD Transformer~\citeyearpar{yao2023spike} & 95.60 & 78.40 \\
QKFormer~\citeyearpar{zhou2024qkformer}  & \underline{96.18} & \underline{81.15} \\
\midrule
\textbf{Lapis-S}  & 96.11 & 81.02 \\
\textbf{Lapis-S(FP)} & \textbf{96.56} & \textbf{81.41} \\
\bottomrule
\end{tabular}
\caption{Classification accuracy on CIFAR-10 and CIFAR-100.Lapis-S denotes the 6-bit deployment model, whereas Lapis-S (FP32) uses full-precision body weights.}
\label{tab:cifar_main}
\end{table}

\subsection{Energy Analysis}




We estimate arithmetic energy using the standard 45\,nm operation-level model. The complete accounting protocol is provided in supplementary material. The Laplacian affinity is realized by membrane leakage over a duration set by $D^{h}_{ij}$. For energy accounting, we conservatively cost each affinity evaluation as one lookup over the $CT+1$ attainable values per head, which upper-bounds a dedicated leakage circuit. Restricting the estimate to the attention path, dense dot-product attention requires 100.01\,mJ per image under the same operation-level model, whereas the Lapis attention path requires 6.92\,mJ, a 14.5$\times$ reduction, and 1.96\,mJ once body weights are quantized to 6 bits. 
Relative to the architecture-matched dense execution, Lapis-L with W32 body weights reduces the estimated energy from 283.15\,mJ to 13.20\,mJ per image, corresponding to a $21.4\times$ reduction before weight quantization. Using 6-bit body weights further reduces the energy to 3.28\,mJ, yielding a $86.3\times$ reduction over dense execution while retaining 83.25\% top-1 accuracy. The two operating points separate the sources of efficiency: single-spike TTFS execution establishes the primary gain, and low-precision weights provide an additional $4.03\times$ reduction.


\subsection{Ablation Study}
\paragraph{Relation function.}
Table~\ref{tab:kernel_ablation} isolates the relation function under an otherwise identical training configuration. The Laplacian relation reaches 96.56\%, remaining within 0.53 points of softmax while outperforming the Gaussian kernel by 5.12 points. The Hamming relation, which retains only binary match between first-spike latencies, reaches 87.92\%, confirming that graded temporal distance is necessary for competitive accuracy. It therefore provides the desired operating point for Lapis, which is recognition accuracy close to dot-product attention, together with a query-key score that uses subtraction, absolute value, and accumulation rather than multiplication. Its $L_1$ form also directly matches the TTFS distance and leakage-based decay.
\begin{table}[t]
\centering
\caption{Kernel ablation on CIFAR-10. All variants use the same Lapis-S backbone and training schedule. Hamming measures the number of channels with unequal first-spike latencies, it retains only binary match information and discards graded temporal proximity. ``Mult.free'' indicates computed without multiplication between query and key vectors.}
\label{tab:kernel_ablation}
\begin{tabular}{lccc}
\toprule
Relation & Distance / score & Mult.free & Top-1 (\%) \\
\midrule
Softmax   & Dot product       & - & 97.09 \\
Gaussian  & Squared $\ell_2$  & - & 91.44 \\
Hamming   & Binary match   & \checkmark & 87.92 \\
Laplacian & $\ell_1$          & \checkmark & 96.56 \\
\bottomrule
\end{tabular}
\end{table}

\paragraph{Hardware-oriented components.}
Table~\ref{tab:deployment_ablation} separates the effects of power-of-two normalization and low-precision weights. Replacing exact row normalization with power-of-two scaling changes top-1 accuracy by 0.31 points, showing that general division can be replaced by shift-based scaling while preserving the learned token relations. Moreover, 6-bit body weights lower the estimate to 3.28\,mJ while retaining 83.25\% accuracy.

\begin{table}[t]
\centering
\small
\setlength{\tabcolsep}{3.8pt}
\renewcommand{\arraystretch}{1.05}
\begin{tabular}{lccrr}
\toprule
Normalization
& Weights
& Top-1
& Energy \\
\midrule
Exact
& W32
& 83.70
& 13.27 \\

Power-of-two
& W32
& 83.39
& 13.20 \\

Power-of-two
& W6
& 83.25
& 3.28 \\
\bottomrule
\end{tabular}
\caption{Ablation of the hardware-oriented components on ImageNet-1K.
All rows use the same Lapis-L backbone. Energy is reported in mJ per
image.}
\label{tab:deployment_ablation}
\end{table}

\section{Discussion}

The experiments support a broader view of spiking attention: spike
timing can define the relation space itself, rather than merely encode
activations for an inherited attention operator. Under a matched
training configuration, the Laplacian relation stays within 0.53 points of dot-product scoring and clearly outperforms the Gaussian alternative, while its $\ell_1$ form aligns directly with TTFS distance and leakage-consistent decay. 
The W32 and W6 operating points also clarify the sources of efficiency. Single spike execution and the Lapis attention path already reduce estimated arithmetic energy by $21.4\times$ before weight quantization, while 6-bit weights provide a
further $4.03\times$ reduction. Together with the exact preservation of query-key distances under TTFS conversion, these results make timing-based relation modeling compatible with existing vision Transformer backbones and efficient deployment.

\section{Conclusion}

We presented Lapis, a Laplacian spiking attention mechanism that forms query-key relations from first-spike latency vectors.
Lapis maps $\ell_1$ temporal distance to affinity through
leakage-consistent decay and replaces general row normalization with
power-of-two scaling. On ImageNet-1K, Lapis-L achieves 83.39\%
top-1 accuracy at an estimated 13.20\,mJ with W32 body weights and
83.25\% at 3.28\,mJ with W6 body weights. The latter corresponds to a $86.3\times$ reduction in estimated
arithmetic energy relative to architecture-matched dense execution.
Results on CIFAR-10 and CIFAR-100 further confirm the effectiveness of
the same low-precision deployment across model scales. These results
establish first-spike timing as an effective and efficient basis for
token relation modeling in spiking vision networks.

\bibliography{main}

\clearpage
\appendix
\section{Arithmetic Energy Estimation}
\label{app:energy}

\subsection{Evaluation Protocol}
We estimate the per-image arithmetic energy under the 45\,nm operation model commonly adopted in prior work on spiking Transformers. We use $4.6$\,pJ for an FP32 multiply and accumulate operation and $0.9$\,pJ for an FP32 accumulation. In Lapis-W6, the linear layers use 6-bit weights, for which we use $0.1$\,pJ per integer accumulation. The temporal relation and value aggregation do not use quantized body weights and remain charged at $0.9$\,pJ per accumulation.

\subsection{Window-Level TTFS Activity}
We measure spike activity directly from each trained checkpoint. Forward hooks are attached to all activation quantizers, and the activity of operator $\ell$ is computed as
\begin{equation}
\rho_\ell
=
\frac{
\sum_m \mathbb{I}\!\left[x_{\ell,m}>0\right]
}{
\sum_m 1
},
\label{eq:window_activity}
\end{equation}
where the sum is taken over all evaluated samples and tensor elements. Each positive integer activation produces exactly one spike during the complete TTFS window. Therefore, $\rho_\ell$ is the fraction of neurons that emit their single spike, rather than a per-timestep firing rate.

For a spike-driven operator transferred from dense operation count $N_{\mathrm{dense}}^{(\ell)}$, the corresponding synaptic
operation count is
\begin{equation}
N_{\mathrm{AC}}^{(\ell)}
=
\rho_\ell N_{\mathrm{dense}}^{(\ell)}.
\label{eq:sop_count}
\end{equation}
No additional factor of $T$ is introduced. This is a direct
consequence of the single-spike TTFS representation. The equivalent average activity per timestep is $\rho_\ell/T$, but it is not used to scale the operation count.

\subsection{Operation Counts}

Consider a model with $L$ Transformer blocks, $N$ tokens,
embedding dimension $D$, $H$ attention heads, and MLP
expansion ratio $R$. The architecture-matched dense
execution uses
\begin{align}
N_{\mathrm{QKV}}      &= 3LND^2, \\
N_{\mathrm{rel}}      &= LN^2D, \\
N_{\mathrm{value}}    &= LN^2D, \\
N_{\mathrm{proj}}     &= LND^2, \\
N_{\mathrm{MLP1}}     &= RLND^2, \\
N_{\mathrm{MLP2}}     &= RLND^2.
\label{eq:dense_counts}
\end{align}
All of these operations are counted as dense
multiply--accumulates for the ANN reference.

For Lapis, the corresponding counts are
\begin{align}
N_{\mathrm{QKV}}^{\mathrm{Lapis}}
    &= \rho_{\mathrm{attn}}\,3LND^2, \\
N_{\mathrm{rel}}^{\mathrm{Lapis}}
    &= LN^2D, \\
N_{\mathrm{value}}^{\mathrm{Lapis}}
    &= \rho_{\mathrm{qkv}}\,LN^2D, \\
N_{\mathrm{proj}}^{\mathrm{Lapis}}
    &= \rho_{\mathrm{out}}\,LND^2, \\
N_{\mathrm{MLP1}}^{\mathrm{Lapis}}
    &= \rho_{\mathrm{mlp}}\,RLND^2, \\
N_{\mathrm{MLP2}}^{\mathrm{Lapis}}
    &= \rho_{\mathrm{hid}}\,RLND^2.
\label{eq:lapis_counts}
\end{align}
Here, $\rho_{\mathrm{attn}}$ is the activity at the attention
input, $\rho_{\mathrm{qkv}}$ is the mean activity measured at
the quantized query, key, and value outputs,
$\rho_{\mathrm{out}}$ is the activity of the spike-coded
attention readout, and $\rho_{\mathrm{mlp}}$ and
$\rho_{\mathrm{hid}}$ are the activities at the two MLP
stages.

The temporal relation is deliberately counted densely over
all token pairs. Each of its $LN^2D$ operations is charged as
an absolute-difference accumulation. We do not apply an
additional sparsity factor to this term. This gives a
conservative estimate and attributes the reduction entirely
to replacing channel-wise multiplication with subtraction,
absolute value, and accumulation.

\subsection{Affinity Lookup and Row Scaling}

After the temporal distance is accumulated, each affinity is
obtained from the finite set of values associated with its
attention head. We conservatively charge one table lookup
for every query--key pair,
\begin{equation}
N_{\mathrm{LUT}} = LHN^2.
\label{eq:lut_count}
\end{equation}
We use $10$\,pJ per lookup, which adds approximately
$0.15$\,mJ to Lapis-L. This treatment is conservative with
respect to a dedicated leakage circuit.

For the dense reference, we do not add a separate cost for
the digital exponential. Consequently, the reported
advantage does not rely on assuming an expensive softmax
implementation.

Exact row normalization and power-of-two normalization use
the same affinity row sums. They differ only in the final
scaling operation. Exact normalization requires
\begin{equation}
N_{\mathrm{div}} = LHN^2
\label{eq:division_count}
\end{equation}
general scaling operations. Charging each such operation at
$4.6$\,pJ adds approximately $0.07$\,mJ to Lapis-L.
Power-of-two normalization replaces these operations with
bit shifts, whose arithmetic cost is treated as negligible in
the reported estimate.

\section{Proofs}
\label{app:proofs}

\subsection{Proof of Proposition~\ref{prop:affinity-properties}}
\begin{proof}
We prove the three properties separately.

\paragraph{(i) Finite affinity set.}
For every channel $c$, both
$t_{i,c}^{Q,h}$ and $t_{j,c}^{K,h}$ belong to
$\{0,\ldots,T\}$. Hence,
\[
\left|t_{i,c}^{Q,h}-t_{j,c}^{K,h}\right|
\in \{0,\ldots,T\}.
\]
Since
\[
D_{ij}^{h}
=
\sum_{c=1}^{C}
\left|t_{i,c}^{Q,h}-t_{j,c}^{K,h}\right|,
\]
the distance $D_{ij}^{h}$ is an integer satisfying
\[
0\le D_{ij}^{h}\le CT.
\]
It therefore belongs to $\{0,\ldots,CT\}$. Using
$\beta_h=\exp(-1/\tau_h)$, the corresponding affinity is
\[
A_{ij}^{h}
=
\exp\left(-\frac{D_{ij}^{h}}{\tau_h}\right)
=
\beta_h^{D_{ij}^{h}}.
\]
Consequently,
\[
A_{ij}^{h}
\in
\left\{\beta_h^r:r=0,\ldots,CT\right\},
\]
which contains at most $CT+1$ values.

\paragraph{(ii) Positivity and order consistency.}
Because $\tau_h>0$ and $D_{ij}^{h}\ge 0$,
\[
0<
\exp\left(-\frac{D_{ij}^{h}}{\tau_h}\right)
\le 1.
\]
Thus, $A_{ij}^{h}\in(0,1]$. Moreover,
\[
A_{ij}^{h}=1
\quad\Longleftrightarrow\quad
D_{ij}^{h}=0
\quad\Longleftrightarrow\quad
\mathbf{t}_{i}^{Q,h}=\mathbf{t}_{j}^{K,h},
\]
where the final equivalence follows from the definiteness of
the $\ell_1$ norm.

The function
\[
f_h(x)=\exp\left(-\frac{x}{\tau_h}\right)
\]
is strictly decreasing on $[0,\infty)$. Therefore, for any
two keys $j$ and $\ell$,
\[
A_{ij}^{h}>A_{i\ell}^{h}
\quad\Longleftrightarrow\quad
D_{ij}^{h}<D_{i\ell}^{h}.
\]

\paragraph{(iii) Latency stability.}
Let
\[
\widetilde{\mathbf{t}}_{i}^{Q,h}
=
\mathbf{t}_{i}^{Q,h}
+
\boldsymbol{\delta}_{i}^{Q,h},
\qquad
\widetilde{\mathbf{t}}_{j}^{K,h}
=
\mathbf{t}_{j}^{K,h}
+
\boldsymbol{\delta}_{j}^{K,h},
\]
and define
\[
\widetilde D_{ij}^{h}
=
\left\|
\widetilde{\mathbf{t}}_{i}^{Q,h}
-
\widetilde{\mathbf{t}}_{j}^{K,h}
\right\|_1.
\]
By the reverse triangle inequality,
\begin{align}
\left|
\widetilde D_{ij}^{h}-D_{ij}^{h}
\right|
&\le
\left\|
\left(
\widetilde{\mathbf{t}}_{i}^{Q,h}
-
\widetilde{\mathbf{t}}_{j}^{K,h}
\right)
-
\left(
\mathbf{t}_{i}^{Q,h}
-
\mathbf{t}_{j}^{K,h}
\right)
\right\|_1
\\
&=
\left\|
\boldsymbol{\delta}_{i}^{Q,h}
-
\boldsymbol{\delta}_{j}^{K,h}
\right\|_1
\\
&\le
\left\|
\boldsymbol{\delta}_{i}^{Q,h}
\right\|_1
+
\left\|
\boldsymbol{\delta}_{j}^{K,h}
\right\|_1
\\
&\le
\eta_Q+\eta_K.
\end{align}

For $f_h(x)=\exp(-x/\tau_h)$,
\[
\left|f_h'(x)\right|
=
\frac{1}{\tau_h}
\exp\left(-\frac{x}{\tau_h}\right)
\le
\frac{1}{\tau_h},
\qquad x\ge 0.
\]
The mean value theorem therefore gives
\begin{align}
\left|
\widetilde A_{ij}^{h}-A_{ij}^{h}
\right|
&=
\left|
f_h\!\left(\widetilde D_{ij}^{h}\right)
-
f_h\!\left(D_{ij}^{h}\right)
\right|
\\
&\le
\frac{1}{\tau_h}
\left|
\widetilde D_{ij}^{h}-D_{ij}^{h}
\right|
\\
&\le
\frac{\eta_Q+\eta_K}{\tau_h}.
\end{align}
This completes the proof.
\end{proof}

\subsection{Proof of Lemma~\ref{lem:pot-approximation}}

\begin{proof}
Recall that
\[
k_i^h
=
\operatorname{round}\!\left(\log_2 Z_i^h\right),
\qquad
\widehat Z_i^h
=
2^{k_i^h}.
\]
Since $k_i^h$ is the nearest integer to $\log_2 Z_i^h$,
\[
\left|
k_i^h-\log_2 Z_i^h
\right|
\le
\frac{1}{2}.
\]
Equivalently,
\[
-\frac{1}{2}
\le
\log_2 Z_i^h-k_i^h
\le
\frac{1}{2}.
\]
Exponentiating with base two gives
\[
2^{-1/2}
\le
2^{\log_2 Z_i^h-k_i^h}
\le
2^{1/2}.
\]
Using $\widehat Z_i^h=2^{k_i^h}$, this becomes
\[
\frac{1}{\sqrt{2}}
\le
\frac{Z_i^h}{\widehat Z_i^h}
\le
\sqrt{2}.
\]

The exact and power-of-two normalized weights satisfy
\[
W_{ij}^{\star,h}
=
\frac{A_{ij}^{h}}{Z_i^h},
\qquad
\widehat W_{ij}^{h}
=
\frac{A_{ij}^{h}}{\widehat Z_i^h}.
\]
Therefore,
\[
\widehat W_{ij}^{h}
=
W_{ij}^{\star,h}
\frac{Z_i^h}{\widehat Z_i^h}.
\]
Because $W_{ij}^{\star,h}>0$, multiplying the preceding
bound by $W_{ij}^{\star,h}$ yields
\[
\frac{1}{\sqrt{2}}W_{ij}^{\star,h}
\le
\widehat W_{ij}^{h}
\le
\sqrt{2}W_{ij}^{\star,h},
\]
which proves the result.

The same identity also gives
\[
\sum_{j=1}^{N}\widehat W_{ij}^{h}
=
\frac{Z_i^h}{\widehat Z_i^h}
\in
\left[
\frac{1}{\sqrt{2}},
\sqrt{2}
\right].
\]
Moreover, every element in a row is multiplied by the same
positive factor. Hence, for any keys $j$ and $\ell$,
\[
\frac{\widehat W_{ij}^{h}}
     {\widehat W_{i\ell}^{h}}
=
\frac{A_{ij}^{h}}{A_{i\ell}^{h}}
=
\frac{W_{ij}^{\star,h}}
     {W_{i\ell}^{\star,h}},
\]
so power-of-two scaling preserves all pairwise ratios and the
ordering of the affinities.
\end{proof}

\section{Training and Quantization Details}
\label{app:training}

\paragraph{Model settings.}
Lapis-S follows the small vision Transformer configuration
and is used for CIFAR-10, CIFAR-100, and all controlled
ablations. Lapis-B and Lapis-L use the corresponding
BEiT-B/16 and BEiT-L/16 backbones on ImageNet-1K. The TTFS coding window is set
to $T=15$ for Lapis-S and Lapis-B, and to $T=20$ for
Lapis-L.

\paragraph{CIFAR training and distillation.}
Lapis-S is trained for 100 epochs with a batch size of 128.
We use AdamW with an initial learning rate of
$1\times10^{-4}$, a weight decay of $0.05$, and a cosine
schedule with five epochs of linear warmup.

The teacher is a BEiT-v2-Large model initialized from the
\texttt{timm} ImageNet checkpoint
\texttt{beitv2\_large\_patch16\_224}. Before student
training, the teacher is adapted separately to CIFAR-10 and
CIFAR-100 for 20 epochs with a batch size of 64 and a
learning rate of $5\times10^{-5}$.

The student objective combines classification, logit
distillation, and feature distillation:
\begin{equation}
\mathcal{L}
=
0.5\mathcal{L}_{\mathrm{CE}}
+
0.5\mathcal{L}_{\mathrm{logit}}
+
0.5\mathcal{L}_{\mathrm{feat}}.
\label{eq:cifar_distillation}
\end{equation}
The logit term uses soft KL divergence with distillation
temperature $T_{\mathrm{KD}}=4$. For feature distillation,
we first apply layer normalization and then minimize the
mean squared error between selected intermediate features.
Student blocks $2$, $5$, $8$, and $11$ are aligned with
teacher blocks $5$, $11$, $17$, and $23$, respectively.
A learned linear projection maps the student feature
dimension from $384$ to the teacher dimension of $1024$.
All relation-function ablations use the same teacher,
distillation losses, and training schedule.

\paragraph{ImageNet fine-tuning.}
The ImageNet models are initialized from pretrained BEiT
backbones and first optimized as quantized networks with
exact row normalization and full-precision body weights.
Lapis-L is fine-tuned for 60 epochs with two warmup epochs,
a global batch size of 162, and an initial learning rate of
$8.5\times10^{-5}$. Training uses distributed data parallel
execution on three GPUs. Lapis-B is fine-tuned for 50 epochs
with a global batch size of 256 and an initial learning rate
of $1.3\times10^{-4}$.

Starting from the exact-normalized checkpoint, we recover
the reported deployment configurations through short
quantization-aware fine-tuning. The power-of-two W32 model
with a spike-coded attention readout is fine-tuned for 10
epochs with a batch size of 108 and a learning rate of
$2\times10^{-5}$. The power-of-two W6 model is fine-tuned
for 10 epochs with the same batch size and learning rate.
This stage jointly enables power-of-two normalization,
spike-coded attention readout, and 6-bit body weights. No
additional warmup is used in either recovery stage. The W6
stage uses a light augmentation recipe with
\texttt{mix\_prob}=0.1.

\paragraph{Activation quantization and TTFS conversion.}
Activations are quantized with learned step-size
quantization into nonnegative integer codes
\[
z\in\{0,\ldots,T\}.
\]
The zero code denotes an inactive neuron. Each positive code
emits exactly one spike at
\[
n=T-z,
\qquad 0\le n<T.
\]
Larger activation codes therefore correspond to earlier
spikes. The encoding is applied elementwise throughout the
network. As established in the main paper, this affine
conversion preserves the query--key $\ell_1$ distance and
therefore leaves the Lapis affinities unchanged.

\paragraph{Weight quantization.}
Lapis-W32 retains full-precision body weights. Lapis-W6
uses 6-bit per-channel quantization for the weight-bearing
linear layers in each Transformer block. These layers include
the query, key, and value projections, the attention output
projection, and the two MLP projections. The patch
embedding and classification head remain in FP32.

In the deployment configurations, the attention output is
quantized into a spike representation before the output
projection. The power-of-two denominator is enabled
together with the matching attention-output quantizer and
weight precision during quantization-aware fine-tuning.
This ensures that both the reported accuracy and the measured
spike activity correspond to the same deployed checkpoint.


\end{document}